\documentclass[runningheads]{llncs}
\usepackage{graphicx}
\usepackage{color}
\usepackage{graphicx}
\usepackage{amsmath,amssymb} % define this before the line numbering.
\usepackage{diagbox}
\usepackage{multirow}
\usepackage{xcolor}
\usepackage{cite}
\usepackage{tabu}
\usepackage{color}
\usepackage{dirtytalk}

\begin{document}
\title{Learning a Robust Society of Tracking Parts using Co-occurrence Constraints}
% Replace with your title

\titlerunning{Society of Tracking Parts}
% Replace with a meaningful short version of your title
%
\author{Elena Burceanu\inst{1, 2, 3}\and
Marius Leordeanu\inst{1, 2, 4}}
%
%Please write out author names in full in the paper, i.e. full given and family names. 
%If any authors have names that can be parsed into FirstName LastName in multiple ways, please include the correct parsing, in a comment to the volume editors:
%\index{Lastnames, Firstnames}
%(Do not uncomment it, because you may introduce extra index items if you do that, we will use scripts for introducing index entries...)
\authorrunning{E. Burceanu and M. Leordeanu}
% Replace with shorter version of the author list. If there are more authors than fits a line, please use A. Author et al.
%

\institute{Bitdefender, Romania \\
\email{eburceanu@bitdefender.com} \and  Institute of Mathematics of the Romanian Academy \and Mathematics and Computer Science, University of Bucharest, Romania \and Automatic Control and Computer Science, University Politehnica of Bucharest\\
\email{marius.leordeanu@cs.pub.ro}	}
\maketitle              % typeset the header of the contribution
\begin{abstract}
Object tracking is an essential problem in computer vision that has been researched for several decades. One of the main challenges in tracking is to adapt to object appearance changes over time and avoiding drifting to background clutter.
We address this challenge by proposing a deep neural network composed of different parts, which functions as a society of tracking parts. They work in conjunction according to a certain policy and learn from each other in a robust manner, using co-occurrence constraints that ensure robust inference and learning. From a structural point of view, our network is composed of two main pathways. One pathway is more conservative. It carefully monitors a large set of simple tracker parts learned as linear filters over deep feature activation maps. It assigns the parts different roles. It promotes the reliable ones and removes the inconsistent ones. 
We learn these filters simultaneously in an efficient way, with a single closed-form formulation, for which we propose novel theoretical properties. The second pathway is more progressive. It is learned completely online and thus it is able to better model object appearance changes. In order to adapt in a robust manner, it is learned only on highly confident frames, which are decided using co-occurrences with the first pathway. Thus, our system has the full benefit of two main approaches in tracking. The larger set of simpler filter parts offers robustness, while the full deep network learned online provides adaptability to change. As shown in the experimental section, our approach achieves state of the art performance on the challenging VOT17 benchmark, outperforming the published methods both on the general EAO metric and in the number of fails, by a significant margin.
\end{abstract}

\section{Introduction}
\label{sec:intro}

Object tracking is one of the first and most fundamental problems that has been addressed in computer vision. While it has attracted the interest of many researchers over several decades of computer vision, it is far from being solved~\cite{vot16, vot17, VOT2014, MOT2014, OTB2015}. The task is hard for many reasons. Difficulties could come from severe changes in object appearance, presence of background clutter and occlusions that might take place in the video. The only ground-truth knowledge given to the tracker is the bounding box of the object in the first frame. Thus, without knowing in advance the properties of the object being tracked, the tracking algorithm must learn them on the fly. It must adapt correctly and make sure it does not jump toward other objects in the background. That is why the possibility of drifting to the background poses on of the main challenges in tracking.

Our proposed model, at the conceptual level, is composed of a large group of different tracking parts, functioning like a society, each with different roles and powers over the final decisions. They learn from each other using certain co-occurrence rules and are monitored according to their reliability. 
The way they function together gives them robustness. From a structural point of view, they are all classifiers within a large deep neural network structure, composed of two pathways, namely the FilterParts and the ConvNetPart pathways (see Figure~\ref{fig:STP}). While the first insures robustness through the co-occurrence of a large number of smaller tracker parts, the second pathway insures the ability to adapt to subtle object changes. The ConvNetPart is fully trained online, end-to-end, and uses as ground-truth high confidence tracker responses that are decided together with the whole society of parts. We will refer to the frames of high confident tracker responses as Highly Confident Frames (HCFs). We provide more details in Section~\ref{subsec:convnetpart_pathway}. Using as ground-truth only a small set of high precision points is also related to the recent work on unsupervised object discovery in video~\cite{ema}. 

Our approach is based on two key insights. One is the organization of the whole tracker into a large group of different types of classifiers, simpler and more complex, at multiple scales and with different levels of depth, as part of a larger neural network structure, that make decisions together based on mutual agreements. The second idea is the usage of co-occurrence  constraints as basis for ensuring robustness, both for online training of the overall tracker, as well as for frame by frame inference. \\

\noindent \textbf{Relation to prior work:}
Existing trackers in the literature differ in terms of type of
target region, appearance model, mathematical formulation and optimization. Objects can be represented by boxes, ellipses~\cite{Kuo2011HowDP}, superpixels~\cite{Wang2011SuperpixelT} or blobs~\cite{Godec2011HoughbasedTO}. The appearance model can be described as one feature set over the region or as an array of features, one for each part of the target~\cite{Forsyth2010ObjectDW,Shu2012PartbasedMT,KwonL09}.

In recent years, trackers based on discriminative correlation filters (DCF), such as MOSSE~\cite{MOSSE} and KCF~\cite{KCF}, achieved the best results on public benchmarks. Newer models like Staple~\cite{Staple}, CCOT~\cite{ccot} and ECO~\cite{eco} provide consistent improvements by adding to the DCF model different components, such as multi-channel feature maps and robust scale estimation~\cite{scalebmvc, scalepami}. CCOT, for instance, proposes to learn continuous convolution parameters by optimizing a function that results from transforming the original feature space into a continuous one and applying onto it the continuous convolutions. While learning the parameters continuously, provides adaptability
to the tracker, overfitting to noise and drifting could pose a threat.
To reduce overfitting, ECO comes with a generative model over training samples. Nevertheless, most recent tracking approaches still suffer from overfitting to background noise, which causes tracker failure. 

A common approach for top trackers in the recent literature is to model object features with deep convolutional networks (CNNs). To address the issue of robustness against background noise in the case of online training of CNNs, the TCNN~\cite{TCNN} algorithm, for example, maintains stability of appearance through a tree structure of CNNs. MLDF~\cite{vot16} uses discriminative multi-level deep features between foreground and background together with a Scale Prediction Network. Another approach, MDNET~\cite{MDNET} is used as starting point for many CNN trackers. For instance, SSAT~\cite{vot16} uses segmentation to properly fit the bounding box and builds a separate model to detect whether the target in the frame is occluded or not. It uses this to consider frames for training the shape segmentation model.

Other line of object tracking research is the development of part-based models, which are more resistant to appearance changes and occlusions. Their multi-part nature gives them robustness against noisy appearance changes in the video. In recent benchmarks however, they did not obtain the top results. For instance, in VOT16~\cite{vot16} challenge, while the number of part-based trackers, such as DPCF~\cite{DPCF}, CMT~\cite{CMT}, DPT \cite{DPT}, BDF~\cite{BDF}, was relatively high (25 \%), the best one of the group, SHCT~\cite{SHCT}, is on the 14th place overall. SHCT~\cite{SHCT} is a complex system using a graph structure of the object that models higher order dependencies between object parts, over time. As it is the case with deep networks, we believe complex systems are prone to overfitting to background noise without a high precision way of selecting their unsupervised online training frames. 

Our proposed model combines the best of two worlds. On one hand it uses a powerful deep convolutional network trained on high confidence frames, in order to learn features that better capture and adapt to object appearance changes. On the other hand, it uses the power of a large group of simpler classifiers that are learned, monitored, added and replaced based on co-occurrence constraints. Our approach is validated by the very low failure rate of our tracker, relative to the competition on the VOT2017 and VOT16 benchmarks. \\

\noindent \textbf{Our main contributions:}
1) Our first contribution is the design of a tracker as a dual-pathway network, with FilterParts and ConvNetPart pathways working in complementary ways within a robust society of tracking parts. FilterParts is more robust to background noise and uses many different and relatively simple trackers learned on top of deep feature activation maps. ConvNetPart is better capable to learn object appearance and adapt to its changes. It employs a deep convolutional network that is learned end to end during tracking using unsupervised high confidence frames for ground-truth.
2) Our second contribution is that every decision made for learning and inference of the tracker is based on robust co-occurrence constraints. Through co-occurrences over time we learn which FilterParts classifiers are reliable or not. Thus we can change their roles and add new ones. Also, through co-occurrences between the vote maps of the two pathways, we decide which frames to choose for training the ConvNetPart path along the way. Last but not least, through co-occurrences we decide the next object center by creating a combined vote map from all reliable parts. \\
3) Our third contribution addresses a theoretical point, in Section~\ref{sec:math}. We show that the efficient closed-form formulation for learning object parts simultaneously in a one sample vs. all fashion is equivalent to the more traditional, but less efficient, balanced one vs. all formulation.

\section{Intuition and motivation}
\label{sec:society}

A tracking model composed of many parts, with different degrees of complexity, could use the co-occurrences of their responses in order to monitor over time, which parts are reliable and which are not. This would provide \textbf{stability}.
They could also be used to train the more complex ConvNetPart pathway only on high-confidence frames on which the two pathway responses strongly co-occur in the same region. Thus, they could provide \textbf{robust adaptability}. 
Last but not least, by taking in consideration only where sufficient parts votes co-occur for the object center, we could also achieve \textbf{robust frame to frame performance}. We discuss each aspect in turn, next: \\

\begin{figure}
	\centering
	\includegraphics[width=0.85\linewidth]{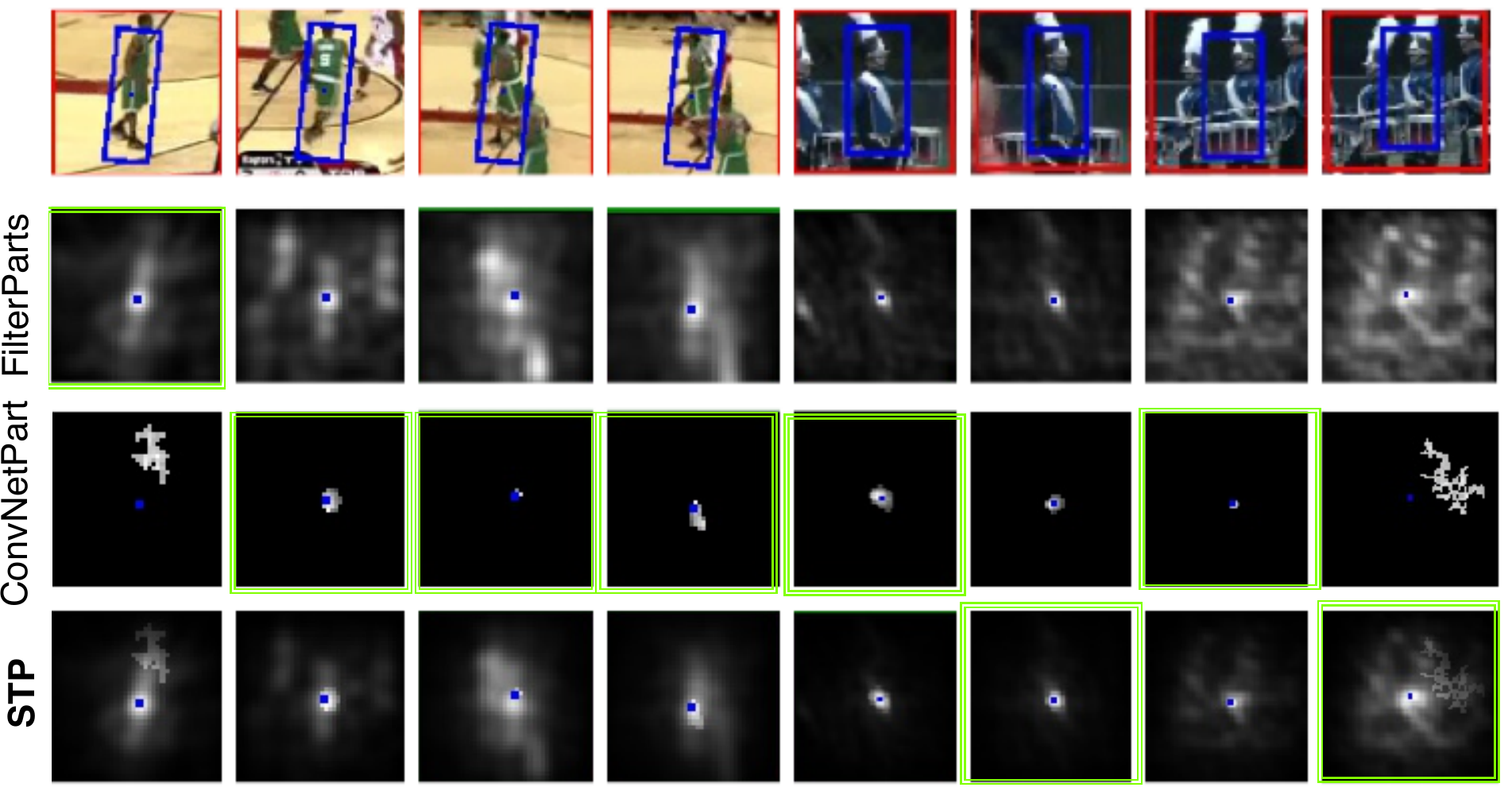}
	\caption{Qualitative comparisons between FilterParts, ConvNetPart and the final (STP) voting maps. Often, in complicated scenarios, the ConvNetPart vote could be of better quality. There are also relatively simple cases where the ConvNetPart activation map look bad, and we need the stability of the FilterParts. The final vote map (STP), provides a more robust maximum. The blue point represent the center of the final vote.}
	\label{fig:vote_samples}
\end{figure}

\noindent \textbf{1) Stability through steadiness:}
A part classifier is a discriminative patch detector (detailed in Section~\ref{subsec:filterparts_pathway}). We consider a part to be reliable if it has showed independently and frequently enough agreement in voting with the majority of the other parts - a statistically robust measure. A certain part is at the beginning monitored as a \textbf{candidate part}, and not used for deciding the next tracker move. It is only after a candidate part's vote for the object center co-occurred frequently enough at the same location with the majority vote, we promote the candidate to become a \textbf{reliable part}. From then on its vote will participate in the final vote map. Tracking parts that display consistent reliable behaviour over relatively long periods of time are promoted to the status of \textbf{gold members} - they will permanently have the right to vote, they cannot be downgraded and will not be monitored. In similar fashion, for the ConvNetPart, we always keep the tracker output from the first frames (=20) in video during the learning updates of the convolutional net. We further ensure robustness by favoring current tracker prediction to be close to the previous one. We use a tracker location uncertainty mask, centered around the previous center location. \\

\noindent \textbf{2) Robust adaptation:} the tracker is able to continuously adapt by adding candidate parts and removing unreliable ones. It also adapts by learning the ConvNetPart on high confidence frames accumulated over time. For object parts along the FilterParts pathway, gaining reliability, loosing it  or  becoming a gold member, can happen only over time. It is the temporal buffer, when tracking parts are monitored, which ensures both stability and the capacity to adapt to new conditions in a robust way. In time, the second pathway has access to a larger and larger set of reliable HCFs that are monitored through co-occurrences between the voted tracker centers of the two pathways. 
By training the net on larger sets of high quality frames we achieve both stability and capacity to adapt to true object appearance changes. As mentioned previously, HCFs used as ground-truth comes from past frames where the center given by the FilterParts alone co-occurred at the same location (within a very small distance) with the one given by the ConvNetPart. In Figure~\ref{fig:hpp_correlation} we show why the distance between the two pathways is a good measure for frame confidence - the strong correlation between the distance between the tracker and the ground-truth and the distance between the centers voted along the two pathways is evident. In Figure~\ref{fig:vote_samples} we also show qualitative results to demonstrate how ConvNetPart and FilterParts could better work together in conjunction, than separately. \\

\noindent \textbf{3) Robust frame to frame tracking:} Each part produces a prediction map for the object center. For the FilterParts pathway, an average vote map is obtained from all reliable parts. That map is then added to the ConvNetPart final vote map, with a strong weight given to the FilterParts pathway. This is the final object center map in which the peak is chosen as the next tracker location. It is thus only through the same strong \textbf{co-occurrences} of votes at a single location that we robustly estimate the next move.

\section{The tracker structure, function and learning}
\label{sec:algo}

\noindent \textbf{Tracker structure:}
At the structural level, the Society of Tracking Parts (STP) has two pathways: FilterParts and ConvNetPart (Figure~\ref{fig:STP}). The first pathway is formed of smaller object parts that are linear classifiers over activation maps, from a pre-learned convolutional net. The ConvNetPart pathway is a deep convolutional net, with the same structure as the first pathway up to a given depth. Now  we present the actual CNNs structures of the two pathways:

The ConvNetPart is a fully convolutional network, where the first part (common as architecture with FilterParts features extractor) has 7 convolutional layers, with 3x3 filters (each followed by ReLU) and 2 maxpooling layers (2x2). It is inspired from the VGG architecture~\cite{vgg}.
The second part, is composed of 4 convolutional layers with 3x3 filters, having the role to gradually reduce the number of channels and computing the segmentation mask for object center prediction. We could have tested with different, more recent architectures, but in our experiments this architecture was strong enough.
\begin{figure}
	\centering
	\includegraphics[width=0.9\linewidth]{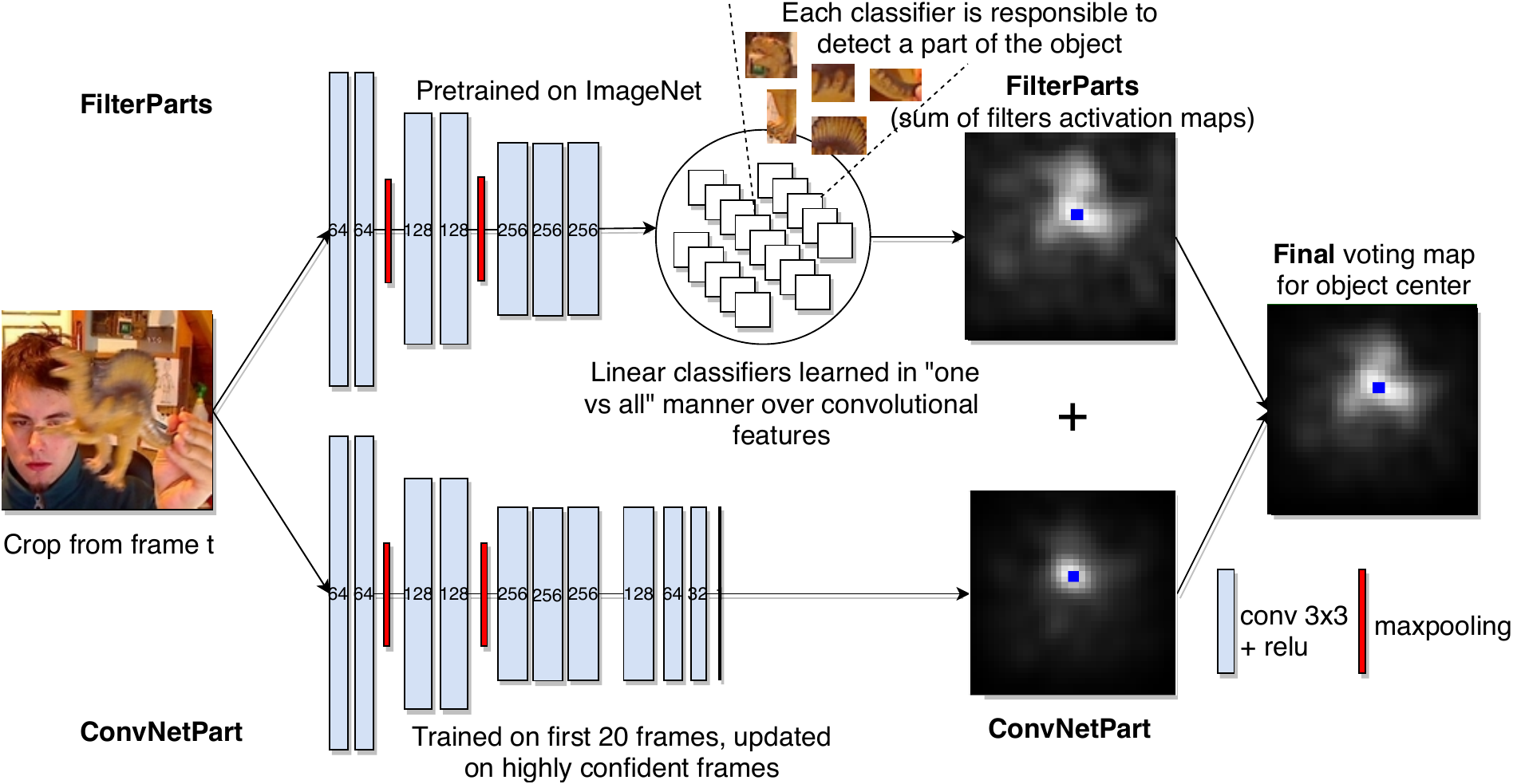}
	\caption{STP overview: The tracker functions as a society of parts. It combines the vote for center maps from all parts over two main pathways, FilterParts and ConvNetPart. The two pathways are learned differently. The FilterParts classifiers once learned are fixed individually but adapt as a group. The ConvNetPart is trained end-to-end with back-propagation over unsupervised tracker outputs from previous highly confident frames (HCFs).}
	\label{fig:STP}
\end{figure}

\noindent \textbf{Tracking by co-occurrences of part votes:}
The tracker always chooses as its next move at time $t$, the place (the center of the bounding box) $l_{t+1}$ where there is the largest accumulation of votes in $P_t$, its final object center prediction map. For each filter part $i$, along the FilterParts pathway, there is an activation map $F_{ti}$, computed as the response of the classifier $c_i$ corresponding to that part over the search region. The activation maps of filter parts are each shifted with the part displacement from object center and added together to form the overall $F_t$. When all filter parts are in strong agreement, all votes from $F_t$ focus around a point. For the second pathway, the object center prediction map $C_t$ is the output of the ConvNetPart network, given the same image crop input as to FilterParts.
After smoothing $F_t$ with a small Gaussian filter, it is added to $C_t$. The final prediction map $P_t$ is then obtained by multiplying pixelwise the linear combination of $C_t$ and $F_t$, with a center uncertainty mask $M_c$, around the center in the previous frame. $M_c$ is a circular soft mask, with exponential decay in weights, as the distance from the previous center prediction increases. Thus, $P_t = (\alpha F_t + (1-\alpha) C_t)\cdot M_c$, where $\cdot$ denotes pixelwise multiplication. 
$M_c$ encourages small center movements at the expense of large, sharp, abrupt ones.
The maximum in $P_t$ is chosen as the next center location $l_{t+1}$. 

\subsection{Learning along the FilterParts pathway}
\label{subsec:filterparts_pathway}
STP chooses in the FilterParts update phase new parts to add as candidates. They are classifiers, of different sizes and locations, represented as linear filters over activation maps of deep features. 
To each part it corresponds a patch, within the tracker's main bounding box. Only patch classifiers that are highly discriminative from the rest are selected. One is considered  discriminative if the ratio between the response on its own corresponding patch (the positive patch) and the maximum response over negatives is larger than a threshold $t_d$. Positive patches are selected from the inside of the bounding box, while (hard) negatives are selected as patches from outside regions with high density of edges. We sample patches from a dense grid (2 pixels stride) of 3 sizes. The small ones will see local appearance and the larger ones will contain some context. A point in grid is covered only by one selected discriminative patch, at one size. The smaller ones have priority and we search the next size for the patch centered in the grid point only if the smaller patch is not discriminative enough. The object box is completely covered when each pixel is covered by any given patch.
A simple budgeting mechanism is added, in order to limit the speed impact. When too many parts of a certain patch size become reliable $>N_{max}$, we remove the new reliable ones which are most similar to older parts, based on simple dot product similarity for the corresponding classifiers.\\

\noindent \textbf{Mathematical formulation for filter parts classifiers:}
\label{sec:math}
We introduce the mathematical formulation for learning the part classifiers in FilterParts.
For a given feature type let $\mathbf{d}_i \in \mathbb{R}^{1 \times k}$ be the $i$-th descriptor, with $k$ real elements, corresponding to an patch window at a certain scale and location relative to the object bounding box. In our case, the descriptor $\mathbf{d}_i$ is a vector version of the specific patch concatenated over all activation map channels over the considered layers of depth in the FilterParts pathway. Our formulation is general and does not depend on a specific level of depth - features could as well be simple pixel values of any image channel. Let then $\mathbf{D}$ be the data matrix, formed by putting all descriptors in the image one row below the other. 

We learn the optimal linear classifier $\mathbf{c}_i$ that separates $\mathbf{d}_i$ from the rest of the patches, according to a regularized linear least squares cost, which is both fast and accurate. Classifier $\mathbf{c}_i$ minimizes the following cost (~\cite{DBLP:books/lib/Murphy12} Ch. 7.5):

\begin{equation}
\label{eq:learning_cost}
\min \frac{1}{n}\|\mathbf{Dc}_i-\mathbf{y}_i\|^2 + \lambda \mathbf{c}_i^\top \mathbf{c}_i.
\end{equation}

In classification tasks the number of positives and negatives should be balanced, according to their prior distributions and the specific classifier used. Different proportions usually lead to different classifiers. In linear least squares formulations weighting differently the data samples could balance learning. \\

\noindent \textbf{Learning with one sample versus all:}
The idea of training one classifier for a single positively labeled data sample has been successfully used before, for example, in the context of SVMs~\cite{malisiewicz2011ensemble}.
When using very few positive samples for training a ridge regression classifier, weighting is applied to balance the data. Here we show that it is possible, when a single positive sample is used, to obtain the same result with a single positive sample without weighting, as if balancing was applied. We show a novel result, that while the magnitude of the corresponding classifier vector is different for the single positive data sample case, its direction remains unchanged w.r.t. the balanced case.  \\

\begin{theorem}
	For any positive weight $w_i$ given to the positive $i$-th sample, when the negative labels considered are $0$ and the positive label is $1$ and all negatives have the same weight $1$, the solution vector to the weighted least squares version of Eq.~\ref{eq:learning_cost} will have the same direction and it might differ only in magnitude. In other words, it is invariant under L2 normalization. 
	\label{single_theorem}
\end{theorem}

\begin{proof}
	Let $\mathbf{c}_i$ be the solution to Eq.~\ref{eq:learning_cost}. At the optimum the gradient vanishes, thus the solution respects the following equality $(\mathbf{D^\top D} + \lambda \mathbf{I}_k)\mathbf{c}_i = \mathbf{D^\top}\mathbf{y}_i$. Since $y_i(i)=1$ and $y_i(j)=0$ for $j \neq i$, it follows that $(\mathbf{D^\top D} + \lambda \mathbf{I}_k)\mathbf{c}_i = \mathbf{d}_i$. Since the problem is convex, with a unique optimum, a point that obeys such an equality must be the solution. In the weighted case, a diagonal weight $n \times n$ matrix $\mathbf{W}$ is defined, with different weights on the diagonal $w_{j}=\mathbf{W}(j,j)$, one for each data sample. In that case, the objective cost optimization in Eq.~\ref{eq:learning_cost} becomes: 

	\begin{equation}
		\label{eq:weighted_learning_cost}
		\min \frac{1}{n}\|\mathbf{W}^{\frac{1}{2}}(\mathbf{Dc}_i-\mathbf{y}_i)\|^2 + \lambda \mathbf{c}_i^\top \mathbf{c}_i.
	\end{equation}
	
	We consider when all negative samples have weight $1$ and the positive one is given $w_i$. 
	Now we show that for any $w_i$, if $\mathbf{c}_i$ is an optimum of Eq.~\ref{eq:learning_cost} then there is a real number $q$ such that $q\mathbf{c}_i$ is the solution of the weighted case. The scalar $q$ exists if it satisfies $(\mathbf{D^\top D} + \mathbf{d}_i \mathbf{d}_i^\top (w_i-1)+\lambda \mathbf{I}_k)q\mathbf{c}_i=w_i\mathbf{d}_i$. And, indeed, it can be verified that $q = \frac{w_i}{1+(w_i-1)(\mathbf{d}_i^\top \mathbf{c}_i)}$ satisfies the required equality. In the \textbf{supplementary material} we have provided a detailed proof.\\
\end{proof}
%\noindent \textbf{Proof:}

\noindent \textbf{Efficient multi-class filter learning:}
The fact that the classifier vector direction is invariant under different weighting of the positive sample suggests that training with a single positive sample will provide a robust and stable separator. The classifier can be re-scaled to obtain values close to $1$ for the positive samples. 
Theorem~\ref{single_theorem} also indicates that we could reliably compute filter classifiers for all positive patches in the bounding box at once, by using a single data matrix $\mathbf{D}$. We form the 
target output matrix $\mathbf{Y}$, with one target labels column $\mathbf{y}_i$ for each corresponding sample $\mathbf{d}_i$. Note that $\mathbf{Y}$ is, in fact, the $\mathbf{I}_n$ identity matrix. We now write the multi-class case of the ridge regression model and finally obtain the matrix of one versus all classifiers, with one column classifier for each tracking part: $\mathbf{C} = \mathbf(D^\top D + \lambda \mathbf{I}_k)^{-1} \mathbf{D}^\top$.
Note that $\mathbf{C}$ is a regularized pseudo-inverse of $\mathbf{D}$. $\mathbf{D}$ contains one patch descriptor per line. In our case, the descriptor length is larger than the number of positive and negative samples, so we use the Matrix Inversion Lemma~\cite{DBLP:books/lib/Murphy12}(Ch. 14.4.3.2) and compute $\mathbf{C}$ in an equivalent form: 

\begin{equation}
\mathbf{C} = \mathbf{D}^\top (\mathbf{D D^\top} + \lambda \mathbf{I}_n)^{-1} . 
\end{equation}

\noindent Now the matrix to be inverted is significantly smaller ($n \times n$ instead of $k \times k$).\\

\noindent \textbf{Reliability states:}
\noindent The reliability of a filter part $i$ is estimated as the frequency $f_i$ at which the maximum activation of a given part is in the neighborhood of the maximum in the final activation $P_{t}$ where the next tracker center $l_{t+1}$ is chosen.
If a part is selected for the first time, it is considered a candidate part. Every $U$ frames, the tracker measures the reliability of a given part, and promotes parts with a reliability larger than a threshold $f_i > p_+$, from candidate state (C) to reliable state (R) and from reliable (R) to gold (G). Parts that do not pass the test $f_i \leq p_-$ are removed, except for gold ones which are permanent.

\subsection{Learning along the ConvNetPart Pathway}
\label{subsec:convnetpart_pathway}
The end output of the ConvNetPart pathway is an object center prediction map, of the same size as the one produced along the FilterParts pathway.
Different from FilterParts, the second pathway has a deeper architecture and a stronger representation power, being trained end-to-end with back-propagation along the video sequence. First, we train this net for the first 20 frames, using as ground-truth the FilterParts center prediction (expected to be highly accurate). Afterwards, the ConvNetPart is considered to be reliable part and it will contribute, through its center prediction, to the final tracker prediction.

From then on, the ConvNetPart will be fine-tuned using as ground-truth the final tracker predictions on highly confident frames (HCFs). This will ensure that we keep the object appearance up to date, and we won't drift in cases of local occlusion or distractors. Results from Table~\ref{tab:hcf_results} supports our decision. \\

\noindent \textbf{Selecting training samples from Highly Confident Frames:}
We call HCF (Highly Confident Frame) a frame on which the distance between FilterParts and ConvNetPart votes for object center prediction is very small. When the two pathways vote almost on the same center location, we have high confidence that the vote is correct. In order to balance efficiently the number of updates with keeping track of object appearance changes, we do the following. First, we accumulate frames of high confidence and second, at regular intervals, we fine tune the network using the accumulated frames. The assumption we made is that on HCFs, our tracker is closer to ground-truth than in the other frames. This is confirmed in Figure~\ref{fig:hpp_correlation}. 11\% of all frames are HCFs. More extensive tests for validating HCF usefulness are described in Section~\ref{sec:experiments}.

\begin{figure}
	\centering
	\includegraphics[width=0.7\linewidth]{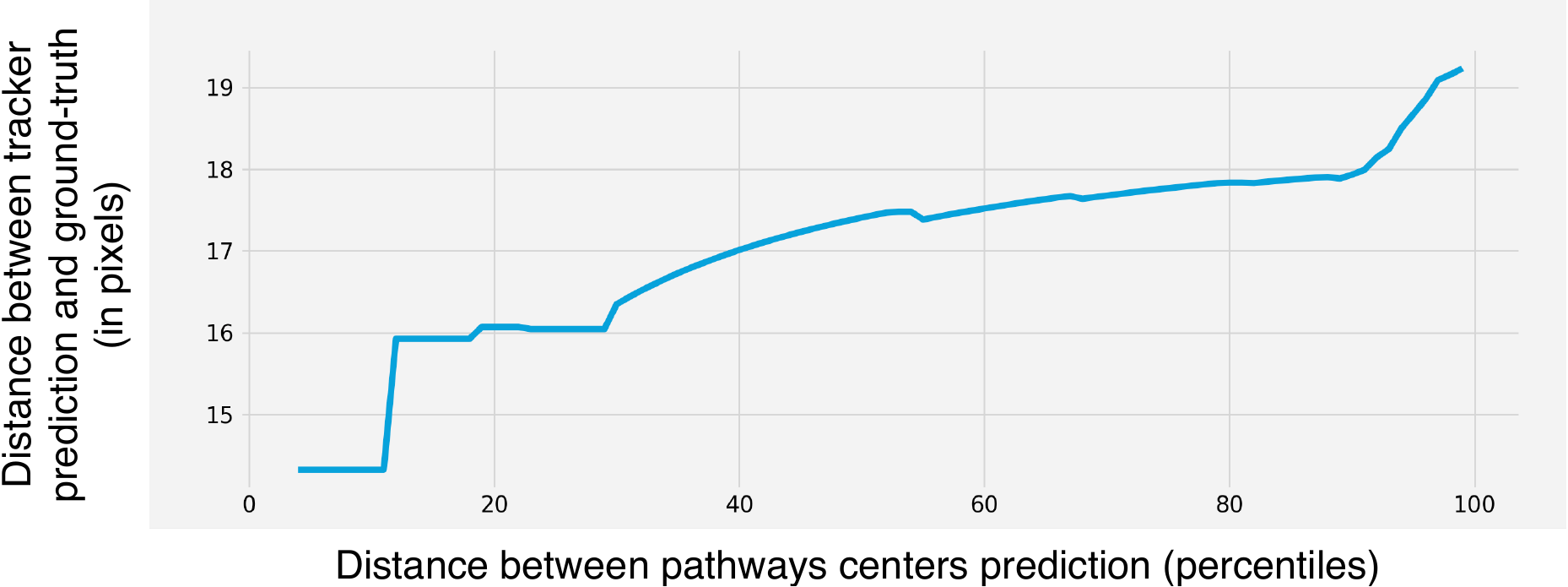}
	\caption{The plot shows the expected distance to ground-truth for a given distance between the centers predicted by the two pathways. As seen, the correlation is strong and is therefore used for selecting in an unsupervised way high confidence frames. We choose HCFs from the first 11\% percentile.}
	\label{fig:hpp_correlation}
\end{figure}

\noindent \textbf{Technical details for training the ConvNetPart:}
For each training frame, we use as input an image crop around the object (Figure~\ref{fig:gt_training_convnetpart}). The ground-truth is given as a segmentation map of the same size, with a circle stamp in the center. We increase robustness and generalization by randomly shifting the image along with its ground-truth - thus we also augment the data by providing two such randomly shifted pairs, per frame. We use the Adam optimizer (Pytorch~\cite{pytorch}), with learning rate $lr$, at first for $k$ epochs on the first $N$(=20) frames, then on $k$ epochs on each update, after each $U$ frames. In the update step, we always use as samples the last $N$ HCFs and the first $N$ frames -  thus we combine the new changes with the initial appearance. The training loss was $MSE = \frac{\sum (x_i - y_i)^2 }{n}$. Note that we did not experiment with many architectures or loss functions, which might have further improve performance.\\ 
\noindent \textbf{Parameters:} we use the following parameters values in all our experiments from Section~\ref{sec:experiments}: $\alpha = 0.6$, $U = 10$ frames, $t_d=1.4$, $p_+ = 0.2$, $p_- = 0.1$, $k = 10$ epochs, $N = 20$ frames, $lr = 1e-5$ and $N_{max} = 200$ parts for each scale size.

\begin{figure}
	\centering
	\includegraphics[width=0.85\linewidth]{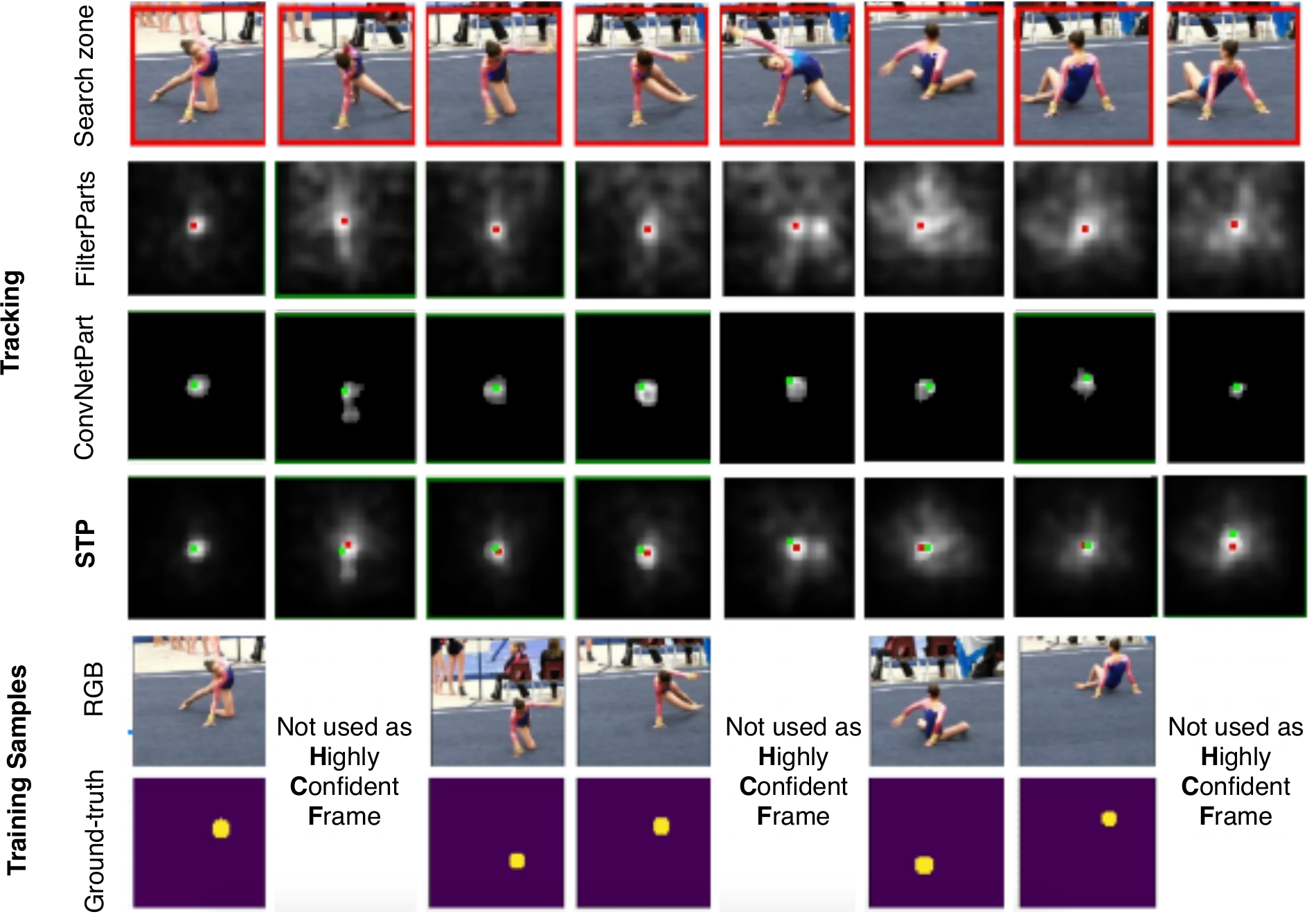}
	\caption{The voting maps for FilterParts, ConvNetPart and the final (STP), respectively. We also show the qualitative view of training samples selection for ConvNetPart. Frame is not Highly Confident if pathways votes centers are distanced. Best seen in colors.}
	\label{fig:gt_training_convnetpart}
\end{figure}

\section{Experimental Analysis}
\label{sec:experiments}

\noindent \textbf{Results on VOT17 and VOT16 benchmarks:} We tested our tracker on the top visual object tracking benchmarks, VOT17~\cite{vot17} and VOT16~\cite{vot16}.
VOT16 contains 60 video sequences, containing many difficult cases of occlusion, illumination change, motion change, size change and camera motion. The VOT17 dataset replaces sequences from VOT16 that were solved by most trackers with new and more difficult ones. For computing the final EAO  evaluation score, VOT setup is re-initializing the tracker when it completely misses the target.

In Table~\ref{tab:vot_results} we present the results after running our tracker through the VOT toolkit. We compared our method against top published tracking methods: ECO~\cite{eco}, CCOT~\cite{ccot}, CFWCR~\cite{CFWCR}, Staple~\cite{Staple}, ASMS~\cite{ASMS}, EBT~\cite{ebt}, CCCT~\cite{CCCT}, CSRDCF~\cite{CSRDCF}, MCPF~\cite{MCPF}, ANT~\cite{ANT}, some with reported results on both benchmarks. Our STP outperforms the current state of the art methods on VOT17, and is in the top three on VOT16. Note that we used the exact same set of parameters on all videos from both VOT17 and VOT16. What distinguishes our tracker the most from the rest is the much lower failure rate ($R$ is 0.76 vs. second best 1.13, on VOT17). We think this is due to the robustness gained by the use of co-occurrence constraints in all aspects of learning and inference, and the dual-pathway structure, with each pathway having complementary advantages. In the \textbf{supplementary material} we present visual results of our tracker on VOT17 and comparisons on different 
challenging cases, as tagged by VOT evaluation. We are in top first or second on 4 out of 5 special cases, while being the first overall as shown in the Table~\ref{tab:vot_results}. VOT16~\cite{vot16} and VOT17~\cite{vot17} identify occlusion as the most difficult case, on which we strongly outperform the others. Next we show how each design choice influenced the strong performance of our tracker. \\

\begin{table}
	\caption{Top published trackers in terms of Expected Average Overlap (EAO), Robustness and Accuracy ($R, R^*$ and $A$). We computed R in two ways: 1) $R$ as initially computed by VOT and also reported by our main competitors, ECO~\cite{eco} and CFWCR~\cite{CFWCR}; and 2) $R^*$, a more complex robustness metric, as currently computed in the VOT benchmark. Note that our tracker outperforms the published methods in terms of both robustness measures $R, R^*$ on both VOT17 and VOT16 by a significant margin. We obtain the state of the art final EAO metric on VOT17 and the 3$^{rd}$ EAO score on VOT16. Our overlap score ($A$) is slightly lower as we did not explicitly learn object shape or mask. Note that we obtained these results with the exact same tracker and parameters for both VOT17 and VOT16. We will make our code available.} 
	
	\centering
	\begin{tabular}{|c| c|c|c|c|  c|c|c|c|}
		\hline
		\multirow{2}{*}{\diagbox[width=10em]{Tracker}{Dataset}}
		& \multicolumn{4}{|c|}{VOT17~\cite{vot17}}& \multicolumn{4}{|c|}{VOT16~\cite{vot16}} \\
		& EAO & $R \downarrow $ & $A \uparrow$ & $R^* \downarrow$  & EAO & $R \downarrow$ & $A \uparrow$ & $R^* \downarrow$\\ 
		\hline
		\textbf{STP} (ours)&  \color{red} \textbf{0.309} & \color{red} \textbf{0.76} & 0.44 & \color{red} \textbf{0.206} & \color{green}0.361 & \color{red} \textbf{0.47} & 0.48 & \color{red} \textbf{0.140} \\
		\hline
		CFWCR~\cite{CFWCR}  & \color{blue}0.303 & \color{green}1.2 & 0.48 & \color{blue}0.267 & \color{red} \textbf{0.39} &\color{green} 0.81 & \color{red}  \textbf{0.58} & - \\
		\hline
		ECO~\cite{eco}&  \color{green} 0.28 & \color{blue} 1.13 & 0.48 & \color{green}0.276 & \color{blue} 0.374 & \color{blue}0.72 & \color{blue}0.54 & \color{blue} 0.200 \\
		\hline
		CCOT~\cite{ccot}  & 0.267 &  1.31 & \color{green}0.49 & 0.318 & 0.331 & 0.85 & \color{green}0.52 & \color{green} 0.238\\
		\hline
		Staple~\cite{Staple}  & 0.169 & 2.5 & \color{red} \textbf{0.53} & 0.688 & 0.295  & 1.35 & \color{blue} 0.54 & 0.378 \\
		\hline
		ASMS~\cite{ASMS} & 0.169  & 2.23 & \color{blue} 0.49 & 0.623 & 0.212   & 1.925 & 0.5 & 0.522 \\
		\hline
		CCCT~\cite{CCCT}  & -  & - & - & -   & 0.223 & 1.83 & 0.442 & 0.461  \\
		\hline
		EBT~\cite{ebt} & - & - & - & - & 0.291  & 0.9 & 0.44 & 0.252 \\
		\hline
		CSRDCF~\cite{CSRDCF}  & 0.256 & 1.368 & 0.491 & 0.356  & -  & -  & - & - \\
		\hline
		MCPF~\cite{MCPF} & 0.248 & 1.548 & 0.510 & 0.427 & -  & - & - & -  \\
		\hline
		ANT~\cite{ANT} & 0.168  & 2.16 & 0.464 & 0.632  & -  & - & - & -  \\
		\hline
	\end{tabular}
	\label{tab:vot_results}
\end{table}

\noindent \textbf{Combining the FilterParts and ConvNetPart pathways:} In Table~\ref{tab:pathways_results} we test the effect of combining the two pathways on the overall tracker. Each pathway is let by itself to guide the tracker. In the "FilterParts only" line, we have results where the first pathway becomes the tracker, with no influence from ConvNetPart ($\alpha = 1$). On the second we show the opposite case, when the tracker is influenced only by ConvNetPart ($\alpha = 0$). In that case the ConvNetPart is trained on the first 20 frames, then continuously updated on its own output, with no influence from the FilterParts pathway.

In general, the FilterParts pathway is more robust and resistant to drifting because it incorporates new information slower, after validating the candidates in time. It is also based on stronger pre-trained features on ImageNet~\cite{imagenet}. It is more stable (lower failure rate) but less capable of learning about object appearance (lower accuracy, as IOU w.r.t  ground-truth). The ConvNetPart pathway is deeper and more powerful, but as it is continuously trained on its own tracker output it is prone to overfitting to background noise, resulting in many failures. 

When using both components, the two pathways work in conjunction and learn from each other using their outputs' co-occurrence constraints. The deeper pathway (ConvNetPart) is learning from the less flexible but more robust pathway (FilterParts). The numbers confirm our intuition and show that the two paths work in complementary, each bringing important value to the final tracker. The boost in performance after combining them is truly significant.\\

\begin{table}
	\caption{In "FilterParts only" experiment, the second pathway is not used at all. In "ConvNetPart only" experiment, we use the FilterParts pathway only for the first 20 frames, to initialize the network, and not use it afterwards. In the absence of high confidence frames selection, the ConvNetPart is trained on each frame, using its own predictions as ground-truth.}
	
	\centering
	\begin{tabular}{|c|   c|c|c|  c|c|c|}
		\hline
		\multirow{2}{*}{\diagbox[width=10em]{Version}{Dataset}}
		& \multicolumn{3}{|c|}{VOT17}& \multicolumn{3}{|c|}{VOT16} \\
		& EAO & $R \downarrow $ & $A \uparrow$  & EAO & $R \downarrow $ & $A \uparrow$ \\
		\hline
		FilterParts only & 0.25 & 0.99 & 0.42 & 0.306 & 0.80 & 0.44 \\
		\hline
		ConvNetPart only  & 0.205 & 2.09 & 0.43 & 0.265 & 1.53 & 0.46\\
		\hline
		Combined  & \textbf{0.309} & \textbf{0.765} & \textbf{0.44} & \textbf{0.361} & \textbf{0.47} & \textbf{0.48}\\
		\hline
	\end{tabular}
	\label{tab:pathways_results}
\end{table}

\noindent \textbf{Using different part roles in FilterParts pathway:}
In this case all filters have one single role. Instead of considering candidates, reliable and gold parts, which ensure stability over time, now all parts added over the sequence have the right to vote at any time. In Table~\ref{tab:filterparts_states} we see that the impact of multiple roles for filter parts, depending on their validation in time is high, bringing a 5\% increase in terms of EAO, comparing to the basic one role for all version.
\begin{table}
	\caption{Impact of different part roles used in FilterParts pathway. Considering roles based on parts credibility over time (candidate, reliable, gold), which is measured using spatial and temporal co-occurrences, is of great benefit to the tracker. It brings an advantage of 5\% in EAO over the vanilla, "one role for all" case.}
	
	\centering
	\begin{tabular}{|c|   c|c|c|  c|c|c|}
		\hline
		\multirow{2}{*}{\diagbox[width=10em]{Version}{Dataset}}
		& \multicolumn{3}{|c|}{VOT17}& \multicolumn{3}{|c|}{VOT16} \\
		& EAO & $R \downarrow $ & $A \uparrow$ & EAO & $R \downarrow $ & $A \uparrow$ \\
		\hline
		One role  & 0.262 &  0.99 &  0.44 & 0.31 & 0.715 & 0.47\\
		\hline
		All roles  & \textbf{0.309} & \textbf{0.765} & \textbf{0.44} & \textbf{0.361} & \textbf{0.47} & \textbf{0.48}\\
		\hline
	\end{tabular}
	\label{tab:filterparts_states}
\end{table}

\noindent \textbf{Learning with Highly Confident Frames on ConvNetPart pathway:} 
In order to better appreciate the value of HCFs in training the ConvNetPart, we have tested it against the cases of training on all frames (all frames are good for training) and that of training only on the first 20 frames (no frame is good, except for the first 20 when the ConvNetPart is initialized).
As we can see in Table~\ref{tab:hcf_results}, the "Full continuous update" regime on all frames is
worst or at most similar in performance with "No update" at all. This shows that the model can overfit very quickly, immediately resulting in drifting (high failure rate). The idea to learn only from Highly Confident Frames is of solid value, bringing a 2\% improvement in the final metric EAO, and a large cut off in failure rate. Even when we randomly select frames to be HCFs, of the same number as in the case of the true HCF measure, we again obtained the same drop of 2\% in performance. These results, along with the statistical correlation between HCF and the ground-truth presented previously in 
Figure~\ref{fig:hpp_correlation} validate experimentally the value of considering only a smaller set of high precision frames for training, even when that set might be just a small portion of all high quality frames.\\

\begin{table}
	\caption{Comparison in performance on VOT17 and VOT16, between updating the ConvNetPart only on Highly Confident Frames (HCF update), not updating it at all (No update), or updating it on every frame (Full update). We mention that in all our experiments we used the top 11$\%$ past frames, in confidence score, to perform training at a given time.}
	
	\centering
	\begin{tabular}{|c|   c|c|c|  c|c|c|}
		\hline
		\multirow{2}{*}{\diagbox[width=10em]{Version}{Dataset}}
		& \multicolumn{3}{|c|}{VOT17}& \multicolumn{3}{|c|}{VOT16} \\
		& EAO & $R \downarrow $ & $A \uparrow$ & EAO & $R \downarrow $ & $A \uparrow$\\
		\hline
		No update & 0.28 & 0.95 & 0.43 & 0.34 & 0.7 & \textbf{0.48} \\
		\hline
		Full update & 0.284 & 0.92 & \textbf{0.44} & 0.327 & 0.66 & 0.46 \\
		\hline
		HCFs update & \textbf{0.309} & \textbf{0.765} & \textbf{0.44}  & \textbf{0.361} & \textbf{0.47} & \textbf{0.48}\\
		\hline
	\end{tabular}
	\label{tab:hcf_results}
\end{table}

\noindent \textbf{Speed:}
The \say{No update} version runs in realtime, at 30fps on GTX TITAN X, for 600 filter parts. The performance of the \say{No update} compared to our best version, the \say{HCFs update}, is only 
2\% lower, in terms of EAO, on both benchmarks, as presented in Table~\ref{tab:hcf_results}.

Our top version, the \say{HCFs update}, runs at 4fps due to the updates of the ConvNetPart, which happen in 5\% of the frames. The computational time needed by these updates depend on the GPU technology we use and is expected to drop in the near future as GPUs are getting faster. The top \say{HCFs update} can run at 30fps if the updates and the tracking are done in parallel, such that when the ConvNetPart update is performed the \say{No update} version continues tracking. The performance of the parallel version drops by about 1\%, situated between the top sequential \say{HCFs update} and the \say{No update} versions.

\section{Conclusions}
\label{sec:concl}

We proposed a deep neural network system for object tracking that functions as a society of tracking parts. Our tracker has two main deep pathways, one that is less flexible but more robust, and another that is less robust but more capable of adapting to complex changes in object appearance. Each part uses co-occurrences constraints in order to keep its robustness high over time, while allowing some degree of adaptability. The two pathways are also combined in a robust manner, by joining their vote maps and picking the locations where their votes co-occurred the most. From a technical point of view, the novelty aspects of our system include: \textbf{1)} the way the classifiers in the FilterParts pathway are learned and ascribed different roles, depending on their degree of reliability. These roles relate to the idea of a society, where some parts are candidates that are being monitored, others are reliable voters, while those who proved their reliability long enough become gold members; \textbf{2)} another novelty aspect represents the way we train the ConvNetPart on high confidence frames only, by selecting for training only those frames where the two different and complementary pathways agree; and \textbf{3)} we provide
a novel theoretical result, which proves that the efficient one sample vs. all strategy employed for learning in the FilterParts path, is stable - it basically gives the same classifier as in the balanced case. In experiments we provide solid validation of our design choices and show state of the art performance on VOT17 and top three on VOT16, while staying on top on both in terms of robustness ($R$ and $R^*$, which measure the failure rate), by a significant margin.

\paragraph{Acknowledgements:} This work was supported in part by UEFISCDI, 
under projects PN-III-P4-ID-ERC-2016-0007 and PN-III-P1-1.2-PCCDI-2017-0734.

\clearpage

\section{Supplementary material}
\noindent In Section \ref{subsec:bbox_estimation} we detailed our simple method for computing the bounding box. Next, in Section \ref{subsec:eval_challenging} we evaluate our tracker on the difficult cases. In Section \ref{subsec:lin_ridge_regress} we elaborate the proof behind our original result, that the weighted solution for our linear classifiers has the same direction as the unbalanced solution. In Section \ref{subsec:matrix_inversion_lemma} we discuss the use of a faster solution for inverting a matrix, adjusted for our case.

\subsection{Bounding Box Estimation per frame}
\label{subsec:bbox_estimation}
We compute the bounding boxes for each frame in a simple manner. For approximating the affine transformation between the original bounding box and the current one, we took a point matching approach. Each part that agreed on the voted center has an expected position w.r.t the center. There is a small displacement between the actual location of parts, given by the maximum response of their corresponding classifier and their expected location w.r.t center. The match between the actual and expected positions gives, through least squares, the affine transformation.

\subsection{Evaluation and comparisons on challenging cases}
\label{subsec:eval_challenging}
Brief comparison with top methods on various difficulty cases, as generated by VOT evaluation. We are in top two on 4 out of 5 cases. Note that VOT16 \cite{vot16} and VOT17 \cite{vot17} identify occlusion as the most difficult case, on which we strongly outperform the others (Figure~\ref{fig:visual_attributes}).

\begin{figure}
	\centering
	\includegraphics[width=0.6\linewidth]{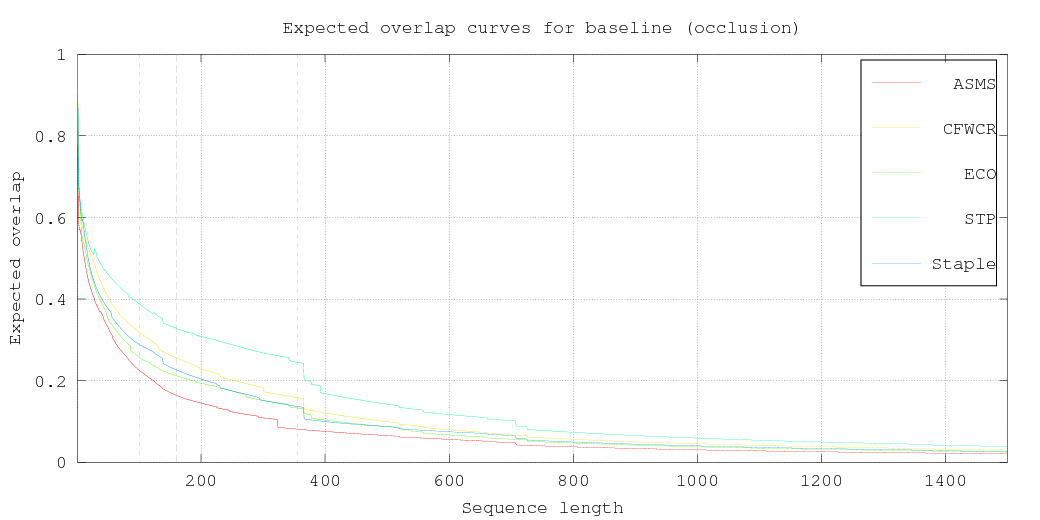}
	\includegraphics[width=0.3\linewidth]{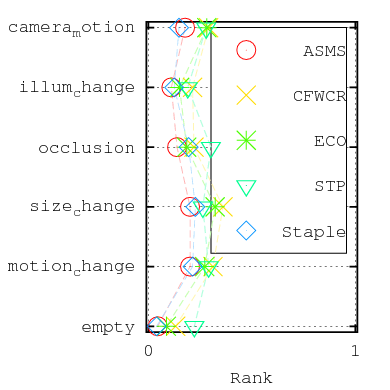}
	\caption{Occlusion EAO (left) and EAO rank for all Visual Attributes (right) for VOT17.}
	\label{fig:visual_attributes}
\end{figure}

\subsection{Fast One-sample vs. All (weighted)}
\label{subsec:lin_ridge_regress}

It is easy to obtain the closed form solution for $c_i$, from linear ridge regression formulation (\cite{DBLP:books/lib/Murphy12} Ch. 7.5), by minimizing the convex cost $\frac{1}{n}\|\mathbf{Dc}_i-\mathbf{y}_i\| + \lambda \mathbf{c}_i^\top \mathbf{c}_i$, we get Eq. \ref{eq:lin_ridge_reg}. It results the well known solution by inverting the positive definite matrix $\mathbf{D^\top D} +\lambda \mathbf{I}_k$.

\begin{equation}
\label{eq:lin_ridge_reg}
(\mathbf{D^\top D} +\lambda \mathbf{I}_k) \mathbf{c}_i = \mathbf{D^\top} \mathbf{y}_i
\end{equation}

In "one vs all" context, we choose 
$y_i^\top$ =
\[
\begin{bmatrix}
0 & 0 & ... & 1 & ... & 0 & 0\\
\end{bmatrix}\], with 1 only on the $i^{th}$ position. So, the multiplication with $y_i$ selects a column form $\mathbf{D}$: $\mathbf{D^\top} \mathbf{y}_i = \mathbf{d}_i$. Eq. \ref{eq:lin_ridge_reg} becomes:

\begin{equation}
\label{eq:simplified_lin_ridge_reg}
(\mathbf{D^\top D} +\lambda \mathbf{I}_k) \mathbf{c}_i  = \mathbf{d}_i
\end{equation}

When building classifiers, the classes should be balanced as numbers of entries. This ensures that the comparison between the activation scores for two different classifiers is valid. In "one vs all", usually the positive class has fewer instances than the negative class. So we needed to use a weighted solution for linear ridge regression \cite{weighted_lrr}, in order to build a balanced classifier for our "part of the object vs others/context" classifiers.

We prove that for a specific form of the weights, the weighting can be applied after computing the simple version (closed form linear regression, Eq. \ref{eq:simplified_lin_ridge_reg}). This is very important for our algorithm, because for the simple ridge regression we need to compute only one matrix inverse for all classifiers in one step, one matrix that all of them will share: ($\mathbf{D^\top D} +\lambda \mathbf{I}_k)^{-1}$ from Eq. \ref{eq:simplified_lin_ridge_reg}. For the weighted case, the closed form solution (as in  \cite{weighted_lrr}) would be different from classifier to classifier:

\begin{equation}
\label{eq:weight_lin_ridge_reg}
(\mathbf{D^\top W}_i \mathbf{D} +\lambda \mathbf{I}_k)\mathbf{\theta}_i= \mathbf{D^\top W}_i \mathbf{y}_i
\end{equation}

The weight matrix for a classifier $\mathbf{W_i}$ has the following form:
\begin{equation}
\label{eq:sparse_w}
\mathbf{W_i} = \mathbf{I}_n + 
\begin{bmatrix}
0 & 0 & 0 & ... & 0 & 0 & 0 \\
0 & 0 & 0 & ... & 0 & 0 & 0 \\
...\\
0  & 0 & ... & w_i &... & 0 & 0\\
...\\
0 & 0 & 0 & ... & 0 & 0 & 0 \\
\end{bmatrix}
= \mathbf{I}_n + \mathbf{W}_{sparse_i}
\end{equation}
with $0$s and  $w_i$ only on $i^{th}$ position on the diagonal, $i$ being the index of the positive patch in data matrix, $\mathbf{D}$.

Replacing Eq.\ref{eq:sparse_w} in Eq.\ref{eq:weight_lin_ridge_reg}, and observing that $\mathbf{D^\top} \mathbf{W}_{sparse_i} = w_i [0|\mathbf{d}_i|0]$, the right hand side becomes: $\mathbf{D^\top W}_i \mathbf{y}_i = \mathbf{D^\top} (\mathbf{I}_n + \mathbf{W}_{sparse_i}) \mathbf{y}_i = \mathbf{D^\top} \mathbf{y}_i + w_i [0| \mathbf{d}_i | 0] \mathbf{y}_i = \mathbf{d}_i + w_i \mathbf{d}_i $. So, for the right term we get:

\begin{equation}
\label{eq:simplified_y}
\mathbf{D^\top W}_i \mathbf{y}_i  = (1 + w_i) \mathbf{d}_i 
\end{equation}

By doing the same operations on the left term:
$\mathbf{D^\top W}_i \mathbf{D} = \mathbf{D^\top} (\mathbf{I}_n + \mathbf{W}_{sparse_i}) \mathbf {D} = \mathbf{D^\top D} + w_i [0|\mathbf{d}_i|0] \mathbf{D} = \mathbf{D^\top D} + w_i \mathbf{d}_i  \mathbf{d}_i^\top$, the Eq. \ref{eq:weight_lin_ridge_reg} can be rewritten:

\begin{equation}
\label{eq:simplified_weight_lin_ridge_reg}
(\mathbf{D^\top D} + w_i \mathbf{d}_i  \mathbf{d}_i^\top +\lambda \mathbf{I}_k)\mathbf{\theta}_i= (1 + w_i) \mathbf{d}_i
\end{equation}

Let $\theta_i = q_i \mathbf{c}_i$, where $\mathbf{c}_i$ is the solution for linear ridge regression (Eq. \ref{eq:simplified_lin_ridge_reg}) and $q_i \in \mathbb{R}$. Then Eq. \ref{eq:simplified_weight_lin_ridge_reg} becomes:
$(\mathbf{D^\top D} + w_i \mathbf{d}_i \mathbf{d}_i^\top+\lambda \mathbf{I}_k) q_i \mathbf{c}_i= (1+w_i)\mathbf{d}_i$. From Eq. \ref{eq:simplified_lin_ridge_reg}, by simplifying terms we obtain
$q_i \mathbf{d}_i + q_i w_i \mathbf{d}_i \mathbf{d}_i ^\top \mathbf{c}_i = (1+w_i)\mathbf{d}_i$. Then, by multiplying at left with $\frac{ \mathbf{d}_i^\top}{||\mathbf{d}_i ||_2^2}$, we get:

\begin{equation}
\label{eq:pre_q}
q_i + q_i w_i  \mathbf{d}_i^\top \mathbf{c}_i = (1 + w_i)
\end{equation}

So, the solution for $q_i$ is ($w_i$ is $n-1$, because in "one vs all" classification, all elements in $\mathbf{D}$ are negative samples, except for one, the $i^{th}$):

\begin{equation}
\label{eq:final_q}
q_i  = \frac{(1 + w_i)} {1 + w_i  \mathbf{d}_i^\top \mathbf{c}_i} = \frac{n} {1 + (n-1)  \mathbf{d}_i^\top \mathbf{c}_i}
\end{equation}

Therefore, we proved that if $\mathbf{c}_i$ is the unique solution of linear ridge regression (since $\mathbf{D^\top D} + \lambda \mathbf{I}_k $ is always invertible, the solution in Eq. \ref{eq:simplified_y} is unique), then $q_i \mathbf{c}_i$ ($q_i$ from Eq. \ref{eq:final_q}) is the unique solution of Eq. \ref{eq:simplified_lin_ridge_reg} ($\mathbf{D^\top D} + w_i \mathbf{d}_i \mathbf{d}_i^\top+\lambda \mathbf{I}_k$ is always invertible, since it is also positive definite).

\subsection{Faster solution with efficient matrix inversion}
\label{subsec:matrix_inversion_lemma}
Consider a general partitioned matrix $\mathbf{M}$
=
\[\begin{bmatrix}
\mathbf{E} & \mathbf{F} \\
\mathbf{G} & \mathbf{H} \\
\end{bmatrix}\], with $\mathbf{E}$ and $\mathbf{H}$ invertible (Matrix Inversion Lemma~\cite{DBLP:books/lib/Murphy12}, Ch. 4.3.4.2). Then the following relation takes place:

\begin{equation}
(\mathbf{E} - \mathbf{F H}^{-1}\mathbf{G})^{-1} \mathbf{F} \mathbf{H}^{-1} = \mathbf{E}^{-1}\mathbf{F}(\mathbf{H} - \mathbf{G}\mathbf{E}^{-1}\mathbf{F})^{-1}
\end{equation}

By making the replacement: $\mathbf{E} = \lambda \mathbf{I}_k$, $\mathbf{H} =  \mathbf{I}_n$, $\mathbf{F} = \mathbf{D}^\top $, $\mathbf{G} = -\mathbf{D}$ ($\mathbf{E}$ and $\mathbf{H}$ are invertible) and rearranging the terms, we obtain ~\cite{DBLP:books/lib/Murphy12} (Ch. 14.4.3.2):

\begin{equation}
\label{eq:matrix_invers_lin_regres}
(\mathbf{D^\top D} +\lambda \mathbf{I}_k)^{-1} \mathbf{D^\top} =  \mathbf{D^\top} (\mathbf{D D^\top} +\lambda \mathbf{I}_n)^{-1}
\end{equation}
We observe that the first term in Eq. \ref{eq:matrix_invers_lin_regres} is part of the closed form solution for the linear regression (without labels $y_i$). So, we can replace it with the one easier to compute. Since the bottleneck here is inverting the positive definite matrix $\mathbf{D^\top D} +\lambda \mathbf{I}_k$ or $\mathbf{D D^\top} +\lambda \mathbf{I}_n$, we will choose the easiest to invert. And this is the smaller one. In our case, $n$ is the number of patches, and $k$ is the number of features in each patch (equal to the patch area in feature space $\times$ number of channels, which is 256). A roughly approximation for $n$ is 500 and approximations for $k$ are $6400 \approx 5 \times 5 \times 256 $, $74000 \approx 17 \times 17 \times 256 $ and bigger for patches of bounding box size.

The second solution for computing the classifier is inverting a matrix two orders of magnitude smaller (as number of elements) than the first solution. So we choose the second part of Eq. \ref{eq:matrix_invers_lin_regres} for the closed form solution.

\clearpage

%===========================================================
\bibliographystyle{splncs04}
\bibliography{egbib}

\end{document}